\documentclass[
  aps,
  prapplied,
  reprint,
  superscriptaddress,
  longbibliography,
  floatfix
]{revtex4-2}

\usepackage[T1]{fontenc}
\usepackage[utf8]{inputenc}
\usepackage{microtype}
\usepackage{amsmath,amssymb,amsthm,mathtools}
\usepackage{bm}
\usepackage{bbm}
\usepackage{mathrsfs}
\usepackage{graphicx}
\usepackage[colorlinks=true,citecolor=magenta,linkcolor=magenta,urlcolor=magenta]{hyperref}

\DeclareMathOperator{\Vol}{Vol}

\DeclareMathOperator{\Tr}{Tr}
\DeclareMathOperator{\rank}{rank}
\DeclareMathOperator{\range}{range}

\newcommand{\E}{\mathbb{E}}

\newcommand{\Cov}{\mathrm{Cov}}
\newcommand{\R}{\mathbb{R}}

\newcommand{\N}{\mathbb{N}}
\newcommand{\1}{\mathbbm{1}}

\allowdisplaybreaks
\newtheorem{theorem}{Theorem}[section]
\newtheorem{lemma}[theorem]{Lemma}
\newtheorem{corollary}[theorem]{Corollary}
\newtheorem{proposition}[theorem]{Proposition}
\newtheorem{assumption}[theorem]{Assumption}

\newtheorem{definition}[theorem]{Definition}

\newcommand{\IPC}{C_{\mathrm{ip}}}
\newcommand{\IC}{C_{\mathrm{i}}}

\begin{document}

\title{Innovation Capacity of Dynamical Learning Systems}

\author{Anthony M.\ Polloreno}
\email{ampolloreno@gmail.com}

\date{\today}

\begin{abstract}
In noisy physical reservoirs, the classical information-processing capacity $C_{\mathrm{ip}}$ quantifies how well a linear readout can realize tasks measurable from the input history, yet $C_{\mathrm{ip}}$ can be far smaller than the observed rank of the readout covariance. We explain this ``missing capacity'' by introducing the innovation capacity $C_{\mathrm{i}}$, the total capacity allocated to readout components orthogonal to the input filtration (Doob innovations, including input-noise mixing). Using a basis-free Hilbert-space formulation of the predictable/innovation decomposition, we prove the conservation law $C_{\mathrm{ip}}+C_{\mathrm{i}}=\mathrm{rank}(\Sigma_{XX})\le d$, so predictable and innovation capacities exactly partition the rank of the observable readout dimension covariance $\Sigma_{XX}\in \mathbb{R}^{\rm d\times d}$. In linear-Gaussian Johnson-Nyquist regimes, $\Sigma_{XX}(T)=S+T N_0$, the split becomes a generalized-eigenvalue shrinkage rule and gives an explicit monotone tradeoff between temperature and predictable capacity. Geometrically, in whitened coordinates the predictable and innovation components correspond to complementary covariance ellipsoids, making $C_{\mathrm{i}}$ a trace-controlled innovation budget. A large $C_{\mathrm{i}}$ forces a high-dimensional innovation subspace with a variance floor and under mild mixing and anti-concentration assumptions this yields extensive innovation-block differential entropy and exponentially many distinguishable histories. Finally, we give an information-theoretic lower bound showing that learning the induced innovation-block law in total variation requires a number of samples that scales with the effective innovation dimension, supporting the generative utility of noisy physical reservoirs.
\end{abstract}

\maketitle

\section{Introduction}
\label{sec:intro}

Analog and stochastic computation increasingly appear as first-class components in modern machine learning. Diffusion models integrate reverse-time SDEs~\cite{sohl2015deep}.
Annealing, Langevin and Hamiltonian methods expose temperature and stochasticity as computational knobs. Specialized hardware, including optical interferometers~\cite{hamerly2019large, li2025photonics}, analog crossbars, coupled oscillators~\cite{todri2024computing}, annealers and coherent Ising machines~\cite{takesue2019large, ghimenti2025geometry, mcmahon2016fully}, implement linear maps and energy flows by exploiting native physical dynamics. In parallel, energy and precision costs per bit motivate architectures that use natural dynamics to perform computations and digitize only where exactness is essential.

A natural setting for studying such architectures is stochastic reservoir computing. A fixed dynamical system with a continuous and high-dimensional latent state is driven by an input and read out linearly. Recent work shows that the concept class accessible to a linear readout can be severely restricted under physical constraints~\cite{polloreno2025restrictions}, by demonstrating that the classical information-processing capacity ($\IPC$)~\cite{dambre2012information}, which quantifies how well a reservoir realizes a family of input-measurable tasks, can fall well below the observable rank of the readout covariance. In this work we explain that apparent ``missing'' capacity. We show that a noisy reservoir, in addition to computing on inputs, also transforms and propagates Doob innovations in the classical signal-processing sense~\cite{wold1938study,wiener1949extrapolation,kolmogorov1941interpolation,kalman1960,kailath1974innovations,bode2006simplified}. Formal definitions of the predictable/innovation decomposition appear in Sec.~\ref{sec:innovation}.

$\IPC$ only scores the input-measurable component and ignores computation devoted to random variables orthogonal to the input filtration. We therefore define an innovation capacity $\IC$ as the total capacity allocated to tasks in the orthogonal complement of the input-measurable subspace. Our main structural result is an exact conservation law:
\begin{equation}
\IPC+\IC=\rank(\Sigma_{XX})\le d,
\end{equation}
where $d$ is the number of readout coordinates (the dimension of $X$, with $\Sigma_{XX}$ the covariance), so whatever $\IPC$ ``goes missing'' in noisy settings reappears as $\IC$. Consequently, the exponential degradation of $\IPC$ in \cite{polloreno2025restrictions} implies an exponentially large lower bound on $\IC$ for physical, stochastic reservoir computers. In Sec.~\ref{sec:innovation} we define $\IC$ in a basis-free way using the $L^2$ Doob decomposition and show that $\IPC$ and $\IC$ are complementary traces on the readout subspace. In Sec.~\ref{subsec:RLC} we show that for linear-Gaussian reservoirs with Johnson-Nyquist noise scaling \cite{johnson1928thermal, nyquist1928thermal, callen1951irreversibility} we obtain a closed-form generalized-eigenvalue shrinkage formula and a monotone temperature tradeoff. In Sec.~\ref{sec:innovation-geometry}, we give an ellipsoid geometry for the predictable/innovation split, showing that a large innovation budget forces extensive block entropy along a trimmed innovation subspace under mild anti-concentration regularity and hence implies many distinguishable histories. Finally, we prove a distribution-free lower bound for learning the innovation-block law in total variation. We show a large innovation dimension implies hardness via an explicit total variation (TV) and Kullback-Leibler (KL) \cite{kullback1951information} packing and Fano's inequality \cite{fano1966transmission,yu1997assouad,tsybakov2009nonparametric,wainwright2019highdim}, which formally proves a certain kind of generative utility provided by physical, stochastic reservoir computers \cite{polloreno2025restrictions}. 

\section{A simple motivating example}
\label{sec:simple}

Following Shannon's classical observation that physically realizable circuits access an exponentially small fraction of Boolean functions~\cite{Shannon1949}, we highlight an analogous phenomenon for noisy circuits and show that under natural physical constraints, noisy dynamics can generate an exponentially large typical set of output histories.

Consider a depth-$L$ layered directed acyclic graph (modeling a circuit) in which each edge is a noisy channel with total-variation (Dobrushin) contraction coefficient $\theta\in[0,1)$~\cite{dobrushin1956central}. Assume bounded fan-in $\le B$ and polynomial (in the input size $n$) width/size (the standard ``physical'' regime \cite{poulin2011quantum}).
Let $\delta_L$ denote the worst-case total-variation sensitivity of the output $X$ to the input $U$:
\begin{equation}
\label{eq:deltaL}
\delta_L \;:=\; \sup_{u,u'}\ \big\|P(X\mid U{=}u)-P(X\mid U{=}u')\big\|_{\mathrm{TV}}.
\end{equation}
A union bound over all input-to-output paths, together with per-edge contraction, yields
\begin{equation}
\label{eq:deltaL-simple}
\delta_L \;\le\; N_L\,\theta^{L}\ \le\ \mathrm{poly}(n)\,(B\,\theta)^{L},
\end{equation}
where $N_L$ is the number of directed input-to-output paths of length $L$ and the final inequality uses $N_L\le \mathrm{poly}(n)\,B^L$.

In the subcritical regime $B\theta<1$, $\delta_L$ decays exponentially in $L$~\cite{kesten1966limit, EvansSchulman1999}, so the output distribution approaches a channel-dependent attractor $\pi$ (e.g.\ uniform in symmetric/no-bias cases). Applying a continuity bound (Fannes-Audenaert) to $H(X\mid U{=}u)$ uniformly in $u$ yields an entropy floor
\begin{equation}
\label{eq:entropy-floor-tight}
H(X\mid U)\ \ge\ H(\pi)\;-\;f_k\!\big(\min\{1,\delta_L\}\big),
\end{equation}
where $f_k(\delta)=\delta\log_2(k{-}1)+h_2(\delta)$ for $k$ output symbols~\cite{audenaert2005continuity} and
\begin{equation}
h_2(\delta)\ :=\ -\delta\log_2\delta\ -\ (1-\delta)\log_2(1-\delta)
\end{equation}
is the binary entropy function.

This closeness lifts to blocks.
If one run is within TV $\delta_L$ of $\pi$, then for $M$ independent runs,
\begin{equation}
\big\|P(X_{1:M})-\pi^{\otimes M}\big\|_{\mathrm{TV}}
\ \le\ \min\{1,\;M\,\delta_L\}.
\end{equation}
If $\pi$ is non-degenerate ($H(\pi)>0$), the typical set of $X_{1:M}$ has cardinality $\approx 2^{M H(\pi)}$.
Thus, even though the space of physically accessible functions is exponentially degraded \cite{Shannon1949}, the explored history set is exponentially large in block length.
The rest of the paper shows that, for noisy reservoirs, this ``large typical set'' phenomenon is general and controlled by a conserved budget that splits $\rank(\Sigma_{XX})$ into Doob-predictable ($\IPC$) and innovation ($\IC$) capacity.

\section{Computing Capacity over Extended Bases}
\label{sec:numerical}

This section describes the finite-sample capacity estimator used in our simulations and illustrates the predictable/innovation budget on a linear RLC circuit and a nonlinear Duffing oscillator with I/Q readout at finite temperature.
Section~\ref{sec:innovation} then gives the basis-free definitions and exact identities.

\subsection{Capacity estimator and sectoral split}
\label{subsec:estimator}

Let $\mathbf{X}\in\R^{n\times d}$ be a matrix of $d$ readout features over $n$ time indices; i.e., row $t$ is $X_t^\top\in\R^d$.
Let $\mathbf{z}\in\R^n$ be a centered (zero-mean) scalar task time series.
In practice we also center each column of $\mathbf{X}$ (or equivalently include an intercept); the formulas below assume centering.

The empirical per-task capacity for $\mathbf{z}$ is given as
\begin{equation}
\label{eq:cap-task}
\begin{aligned}
C_{\mathrm{ip}}\!\left(\mathbf{X},\mathbf{z}\right)
&=
1-\frac{\min_{w\in\R^d}\sum_{t=1}^n\!\bigl(z_t-w^\top X_t\bigr)^2}{\sum_{t=1}^n z_t^2}
\\
&=
\frac{\mathbf{z}^\top \mathbf{X}\,(\mathbf{X}^\top \mathbf{X})^{+}\mathbf{X}^\top \mathbf{z}}{\mathbf{z}^\top \mathbf{z}},
\end{aligned}
\end{equation}
where $(\cdot)^+$ is the Moore-Penrose inverse (useful when $\mathbf{X}^\top\mathbf{X}$ is ill-conditioned or rank-deficient \cite{polloreno2023note, hu2023tackling}).

To estimate innovation-related capacities in noisy systems, we use the Doob decomposition with respect to the input history, which yields the innovation residual associated with the reservoir noise history \cite{doob1953}. Operationally (in simulation, or in experiments with repeated trials), we fix an input realization and average the readout across independent noise realizations:
\begin{equation}
\label{eq:trial-avg}
\langle X_t\rangle \ \approx\ \E[X_t\mid \mathscr F^{\rm in}_t].
\end{equation}
In hardware, this corresponds to repeating the same injected input waveform across trials under stable operating conditions and averaging the resulting readouts.
Because the reservoir is causal and the exogenous fluctuations are treated as independent of the input, holding the entire input waveform fixed across trials estimates the same conditional mean as conditioning on the input history up to time $t$.

We then define the (Doob) innovation residual
\begin{equation}
\label{eq:DeltaX-def}
\Delta X_t \;:=\; X_t-\langle X_t\rangle.
\end{equation}
This $\Delta X_t$ is, by construction, orthogonal in $L^2$ to input-measurable tasks at time $t$.
In additive linear reservoirs $\Delta X_t$ is ``noise-only,'' while in nonlinear/multiplicative reservoirs it also contains input$\times$noise mixing (while remaining orthogonal to the input $\sigma$-algebra at time $t$).

To approximate the basis-free split in Sec.~\ref{sec:innovation} with finite task sets, we construct orthonormal task blocks and sum their empirical capacities over an input-measurable task family (the predictable sector), built from delayed input polynomials as in~\cite{dambre2012information,kubota2021}, an innovation family built from $\Delta X$ (innovation sector) and mixed tasks built from products of input tasks and $\Delta X$ tasks.
Each task block is centered, projected onto the current orthogonal complement and whitened so that summed capacities are stable and sector-wise additive up to sampling error. In direct analog to the information processing capacity $C_{\rm ip}$ in Eq.~\eqref{eq:cap-task}, in the following examples we name the sum of the innovation and mixed tasks the innovation capacity and denote it $C_i$.

\subsection{Linear (RLC) reservoir}
\label{subsec:RLC}

To start, we consider a stable linear state-space model driven by input $u$ and internal noise $\eta$,
\begin{equation}
\label{eq:rlc-ct}
\dot{x}(t)=A x(t)+B_s u(t)+B_n \eta(t),\qquad  X(t)=C x(t)\in\R^{d}.
\end{equation}
In steady state, the readout covariance decomposes additively as $\Sigma_{XX}(T)=S+N(T)$,
\begin{equation}
\label{eq:SplusN}
S:=C P_s C^\top,\quad N(T):=C P_n(T) C^\top,
\end{equation}
where $P_s$ and $P_n(T)$ solve continuous-time Lyapunov equations \cite{parks1992lyapunov}.
Under Johnson-Nyquist scaling, the innovation covariance inflates linearly with temperature, $N(T)=T\,N_0$ for some $N_0\succeq 0$. In this additive setting the predictable/innovation split has a closed form:
\begin{equation}
\begin{aligned}
\label{eq:rlc-traces}
\IPC(T)&=\Tr\!\big(S\,\Sigma_{XX}(T)^{+}\big),\\
\IC(T)&=\Tr\!\big(N(T)\,\Sigma_{XX}(T)^{+}\big),
\end{aligned}
\end{equation}
and $\IPC(T)+\IC(T)=\rank\Sigma_{XX}(T)$.
\begin{proposition}[Johnson-Nyquist temperature tradeoff]
\label{prop:temperature-clean}
Assume $\Sigma_{XX}(T)=S+T\,N_0$ with $S,N_0\succeq 0$ and $T\ge 0$.
Let $r:=\rank\Sigma_{XX}(T)$, assumed constant on an interval $T\in[T_1,T_2]$, and let $r_S:=\rank S$.
Then there exist nonnegative finite generalized eigenvalues $\lambda_1,\dots,\lambda_{r_S}\in[0,\infty)$ of the \emph{symmetric pencil} $(N_0,S)$ (i.e.\ scalars $\lambda$ for which $N_0v=\lambda Sv$ has a nonzero solution $v$ with $Sv\neq 0$) such that for all $T\in[T_1,T_2]$,
\begin{equation}
\label{eq:temp-tradeoff}
\IPC(T)=\sum_{k=1}^{r_S}\frac{1}{1+T\lambda_k},\qquad
\IC(T)=(r-r_S)\;+\;\sum_{k=1}^{r_S}\frac{T\lambda_k}{1+T\lambda_k}.
\end{equation}
In particular, $\IPC(T)$ is nonincreasing and $\IC(T)$ is nondecreasing in $T$, and $\IPC(T)+\IC(T)=r$ on $[T_1,T_2]$.
\end{proposition}

\begin{proof}
Work on the active subspace $\mathcal R_X=\range(\Sigma_{XX}(T))$ (dimension $r$), on which $\Sigma_{XX}(T)$ is invertible and $\Sigma_{XX}(T)^+$ agrees with its inverse.
Consider the generalized eigenproblem for the symmetric pencil $(S,\Sigma_{XX}(T))$ on $\mathcal R_X$:
\begin{equation}
S v=\gamma\,\Sigma_{XX}(T)\,v,\qquad v\in\mathcal R_X.
\end{equation}
The nonzero eigenvalues $\gamma_1,\dots,\gamma_{r_S}$ are the positive eigenvalues of $\Sigma_{XX}(T)^{+}S$ and satisfy $\gamma_k\in(0,1]$.
Taking traces gives
\begin{equation}
\IPC(T)=\Tr\!\big(S\,\Sigma_{XX}(T)^{+}\big)=\Tr\!\big(\Sigma_{XX}(T)^{+}S\big)=\sum_{k=1}^{r_S}\gamma_k.
\end{equation}
For any eigenpair $(\gamma,v)$ with $\gamma>0$, rearranging
\begin{equation}
S v=\gamma(S+TN_0)v
\end{equation}
yields
\begin{equation}
(1-\gamma)\,S v=\gamma T\,N_0 v,
\end{equation}
so $v$ also satisfies $N_0 v=\lambda S v$ with $\lambda=\frac{1-\gamma}{\gamma T}\ge 0$.
Conversely, if $N_0 v=\lambda S v$ with $Sv\neq 0$, then
\begin{equation}
S v = \frac{1}{1+T\lambda}\,(S+TN_0)v=\frac{1}{1+T\lambda}\,\Sigma_{XX}(T)v,
\end{equation}
so $\gamma=\frac{1}{1+T\lambda}$.
Thus the $r_S$ positive eigenvalues of $\Sigma_{XX}(T)^{+}S$ take the shrinkage form $\gamma_k=(1+T\lambda_k)^{-1}$ for the $r_S$ finite generalized eigenvalues $\{\lambda_k\}$ of $(N_0,S)$, and the claimed formula for $\IPC(T)$ follows.

Finally, on $\mathcal R_X$ we have $\Tr(\Sigma_{XX}(T)\Sigma_{XX}(T)^+)=r$, so
\begin{equation}
\begin{aligned}
\IC(T)&=\Tr\!\big(TN_0\,\Sigma_{XX}(T)^{+}\big)
\\&=\Tr\!\big((\Sigma_{XX}(T)-S)\Sigma_{XX}(T)^{+}\big)
\\&=r-\IPC(T),
\end{aligned}
\end{equation}
and rewriting $r-\IPC(T)$ gives \eqref{eq:temp-tradeoff}.
Monotonicity follows by differentiating the scalar shrinkage factors.
\end{proof}

The RLC circuit shown in Fig.~\ref{fig:budget} is a series RLC oscillator
\begin{equation}
\dot q = i,\qquad
L\dot i = -Ri - \frac{1}{C_{\rm cap}}q + \alpha_s u + \alpha_n \eta,
\end{equation}
for current $i$, capacitance $C_{\rm cap}$, charge $q$, noise strength $\alpha_n$, drive strength $\alpha_s$ and with drive $u$ and thermal source $\eta$ that enter additively through the inductor/current equation (i.e.\ a voltage-drive channel in the standard circuit interpretation).
This is a state space model with state variable $(q,i)$, but we read out the measured voltage across the capacitor.
When sweeping temperature $T$ while keeping the input statistics fixed, the signal covariance is independent of $T$, while the innovation covariance scales linearly because $\eta\sim\mathcal N(0,\gamma T)$ for a scalar $\gamma$, giving $N(T)=T\,N_0$.
Therefore the RLC experiment satisfies the assumption $\Sigma_{XX}(T)=S+T\,N_0$ of Proposition~\ref{prop:temperature-clean} and we see excellent agreement between the theory (solid lines) and simulation (points).

\begin{figure}[t]
  \centering
  \includegraphics[width=.85\linewidth]{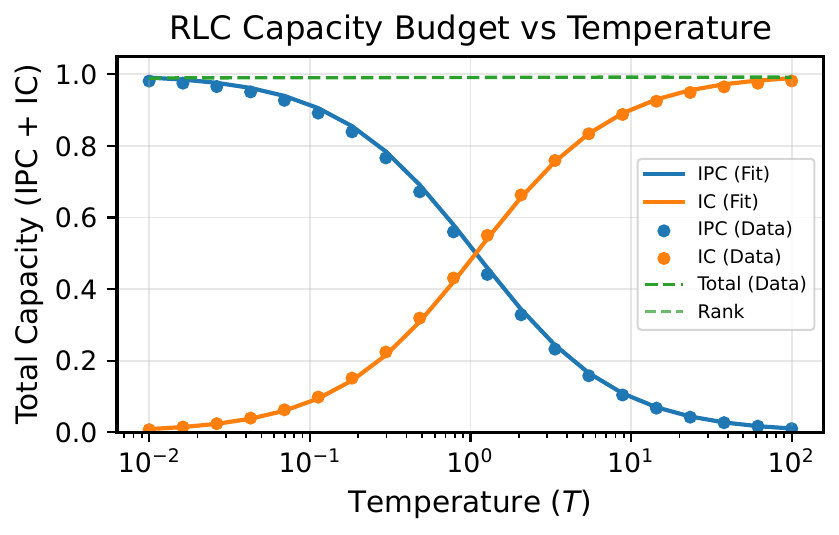}
  \caption{RLC temperature sweep: data-driven $\IPC$ (blue $\cdot$) and $\IC$ (orange $\cdot$) versus analytic predictions (solid) from Proposition~\ref{prop:temperature-clean}. The sum tracks $\rank \Sigma_{XX}(T)$.}
  \label{fig:budget}
\end{figure}

\subsection{Duffing reservoir}
\label{subsec:duffing-theory}

\begin{figure*}[t]
  \centering
  \begin{minipage}[t]{0.49\textwidth}
    \centering
    \includegraphics[width=\linewidth]{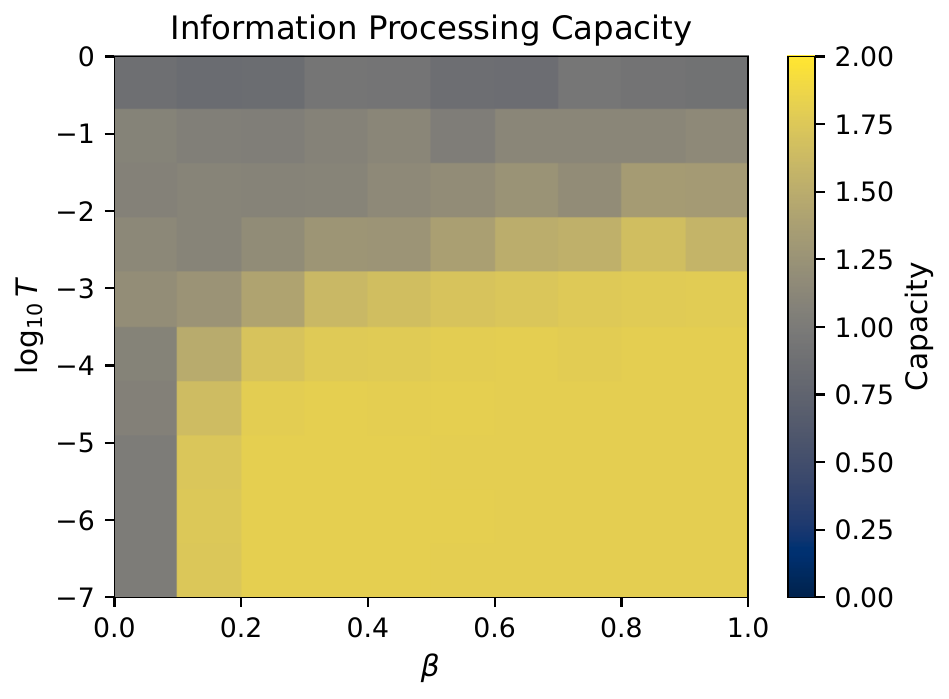}
    \vspace{0.6ex}
    \begin{minipage}[t]{0.49\linewidth}
      \centering
      \includegraphics[width=\linewidth]{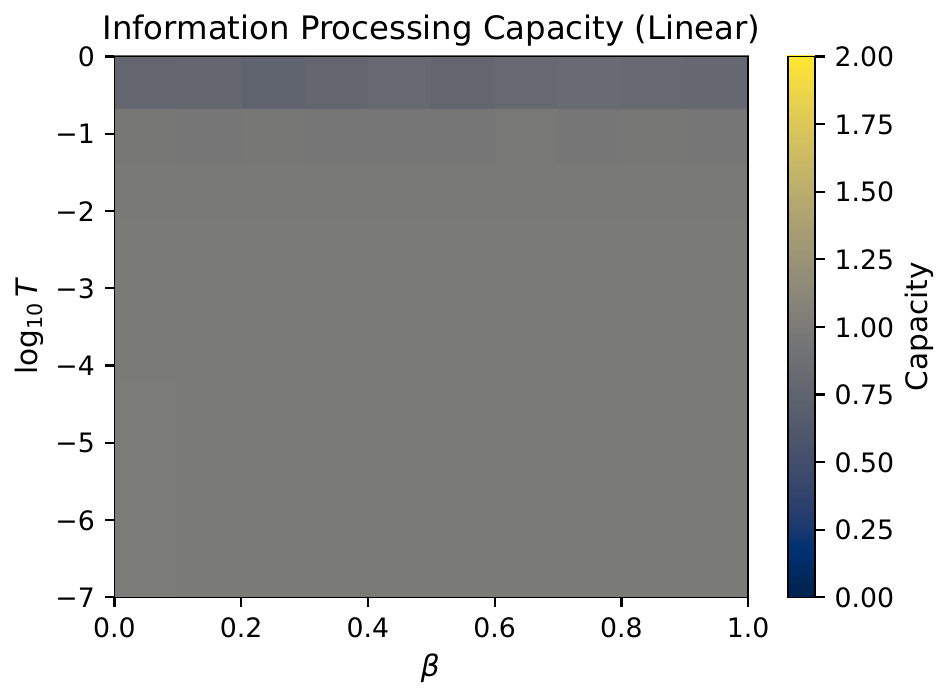}
    \end{minipage}\hfill
    \begin{minipage}[t]{0.49\linewidth}
      \centering
      \includegraphics[width=\linewidth]{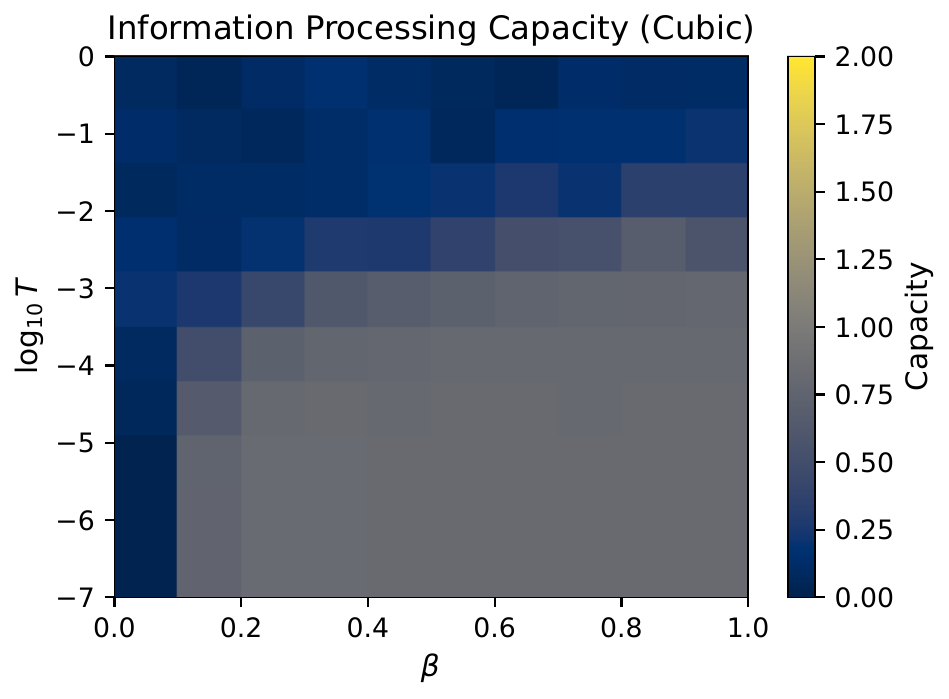}
    \end{minipage}
  \end{minipage}\hfill
  \begin{minipage}[t]{0.49\textwidth}
    \centering
    \includegraphics[width=\linewidth]{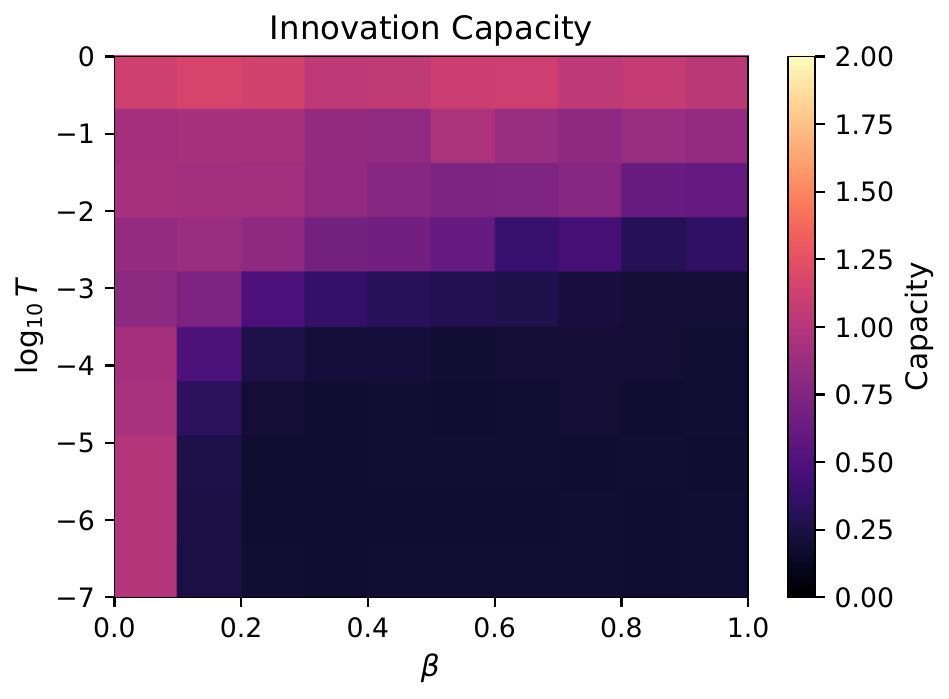}
    \vspace{0.6ex}
    \begin{minipage}[t]{0.49\linewidth}
      \centering
      \includegraphics[width=\linewidth]{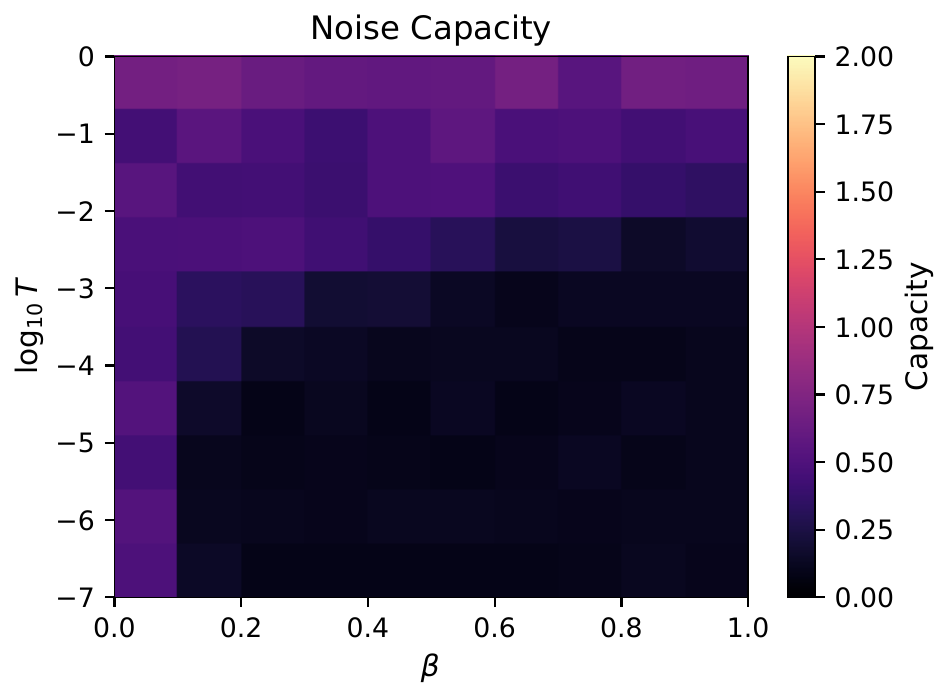}
    \end{minipage}\hfill
    \begin{minipage}[t]{0.49\linewidth}
      \centering
      \includegraphics[width=\linewidth]{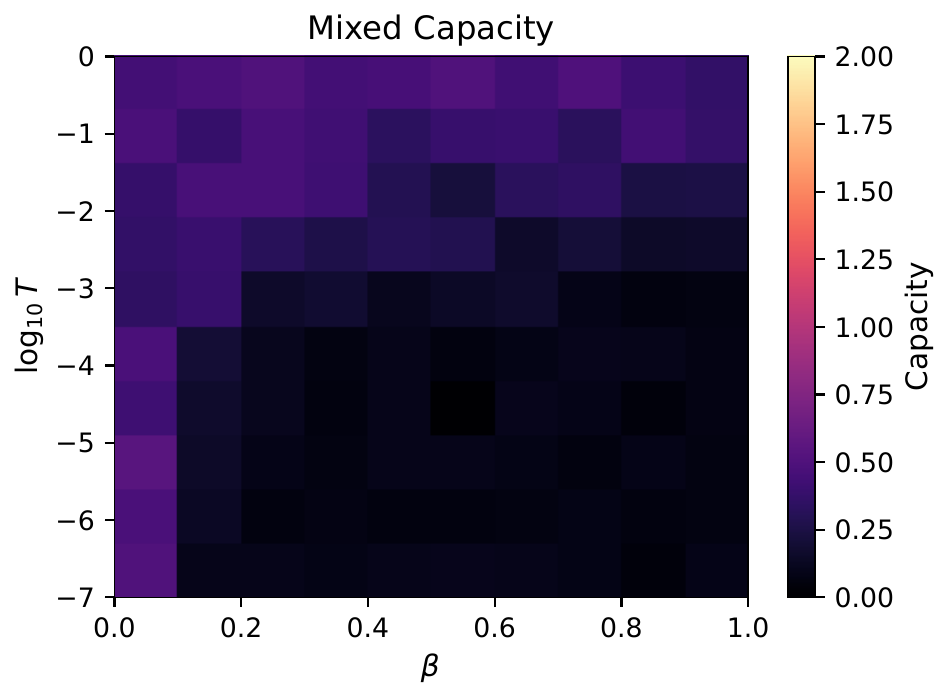}
    \end{minipage}
  \end{minipage}
  \caption{Simulated capacities over $(\beta,T)$ using the demodulate-LPF protocol \eqref{eq:duff-demod}-\eqref{eq:duff-readout} for a Duffing oscillator.
  Top: total IPC and total innovation.
  Bottom: IPC constituents (linear/cubic) and innovation constituents (noise/mixed).}
  \label{fig:duffing-heatmaps}
\end{figure*}

We now illustrate the same budget picture in the nonlinear setting of a damped Duffing oscillator \cite{duffing1918} driven near a carrier $\omega$ and read out in baseband I/Q.
The organizing idea is again to split the readout covariance into a Doob-predictable portion driven by the input history and an innovation portion driven by intrinsic fluctuations.
For the Duffing oscillator, this split is convenient after two standard signal-processing steps: demodulation to baseband and an adiabatic (slow-envelope) approximation.
In that regime the Duffing nonlinearity generates a dominant cubic correction, so the predictable response is well captured by linear and cubic deterministic sectors.

Specifically,
\begin{equation}
\label{eq:duff-ode}
\ddot x + \delta\,\dot x + \alpha\,x + \beta\,x^3
= \alpha_s\, u(t)\cos(\omega t)\;+\;\alpha_n\,\sqrt{T}\,\eta(t),
\end{equation}
where $\alpha_s$ and $\alpha_n$ are scalar signal and noise coefficients, $\eta(t)$ is a zero-mean, unit-intensity stationary noise (idealized as Gaussian white noise) and the explicit $\sqrt{T}$ makes Johnson-Nyquist scaling transparent at baseband.
We assume a single-well regime, weak nonlinearity ($|\beta|$ small) and small detuning from $\omega$.

Define the complex baseband demodulation operator
\begin{equation}
\label{eq:duff-demod}
(\mathcal D_\omega y)(t)
\;:=\;
\mathrm{LPF}_\Omega\!\left( y(t)e^{-i\omega t}\right),
\end{equation}
with a low-pass filter (LPF) cutoff $\Omega\ll\omega$.
We use the baseband envelope $A(t)$ and I/Q readout $X(t)$:
\begin{equation}
\label{eq:duff-readout}
A(t):=2\,(\mathcal D_\omega x)(t),\qquad
X(t):=\begin{bmatrix}\Re A(t)\\ \Im A(t)\end{bmatrix}\in\R^{2}.
\end{equation}

A standard averaging/multiple-scales \cite{nayfeh1979nonlinear,kevorkian1996multiple} argument yields the adiabatic/slow-envelope equation
\begin{equation}
\label{eq:duff-env}
\dot A
=
\Big(-\tfrac{\delta}{2}+i\Delta\Big)A
\;+\;
i\mu |A|^2A
\;+\;
\kappa\, u(t)
\;+\;
\zeta_T(t),
\end{equation}
where $\Delta$ is the detuning, $\mu=\tfrac{3\beta}{8\omega}$, $\kappa=\tfrac{\alpha_s}{2\omega}$ (up to a phase convention) and $\zeta_T$ is the demodulated innovation obtained by applying $\mathcal D_\omega$ to $\alpha_n\sqrt{T}\,\eta$. Under the effective Johnson-Nyquist assumption, $\zeta_T(t)=\sqrt{T}\,\zeta_1(t)$ with $\zeta_1$ a unit-temperature innovation.

Let $\Sigma_{XX}(T,\beta):=\Cov(X(t))$ denote the baseband I/Q covariance (in stationarity, or over a long measurement window). Relative to the (continuous-time) input filtration $\mathscr F^{\rm in}_t:=\sigma(u(s):s\le t)$, the law of total covariance gives $\Sigma_{XX}=S+N$ with
\begin{equation}
\begin{aligned}
S(\beta)&:=\Cov\!\Big(\E[X(t)\mid \mathscr F^{\rm in}_t]\Big),\\
N(T,\beta)&:=\E\!\Big[\Cov(X(t)\mid \mathscr F^{\rm in}_t)\Big].
\end{aligned}
\end{equation}
In the demodulated adiabatic regime, the leading temperature dependence enters through innovation power, suggesting the baseline
\begin{equation}
\label{eq:duff-add}
\Sigma_{XX}(T,\beta)\approx S(\beta)+T\,N_1(\beta),
\end{equation}
where $N_1(\beta)$ is the unit-temperature innovation covariance.

At larger $(\beta,T)$, nonlinear signal-noise mixing generates additional covariance beyond $T\,N_1$.
To capture the leading next-order effect while keeping the model PSD and low-dimensional, we use an isotropic inflation term, with $\bar N(\beta):=\tfrac{1}{2}\Tr\,N_1(\beta)$:
\begin{equation}
\label{eq:duff-add-fit}
\Sigma_{XX}(T,\beta)\;\approx\;
S(\beta)+T\,N_1(\beta)
\;+\;
|\beta|^3\; g(T;\beta)\;\bar N(\beta)\,I_2,
\end{equation}
with a nonnegative polynomial parameterization
\begin{equation}
\label{eq:duff-gT}
g(T;\beta)\;=\;\sum_{k\ge 1} a_k(\beta)\,T^k,\qquad a_k(\beta)\ge 0.
\end{equation}
This is a controlled proxy that captures average in-band covariance inflation from cubic mixing without claiming to resolve anisotropy or detailed fluctuation-dissipation structure outside the adiabatic regime. Figure~\ref{fig:duffing-heatmaps} shows simulated linear, cubic and innovation capacities over $(\beta,T)$ and Fig.~\ref{fig:duff-covfit} validates the covariance-fit model \eqref{eq:duff-add-fit}-\eqref{eq:duff-gT} against simulation.

\begin{figure*}[t]
  \centering
  \begin{minipage}[t]{0.48\textwidth}
    \centering
    \includegraphics[width=\linewidth]{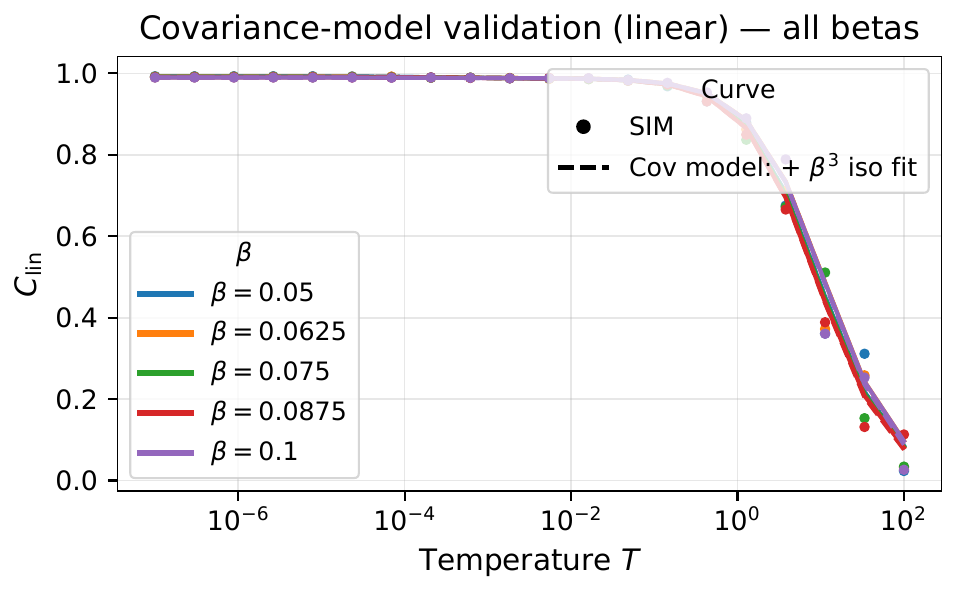}
  \end{minipage}\hfill
  \begin{minipage}[t]{0.48\textwidth}
    \centering
    \includegraphics[width=\linewidth]{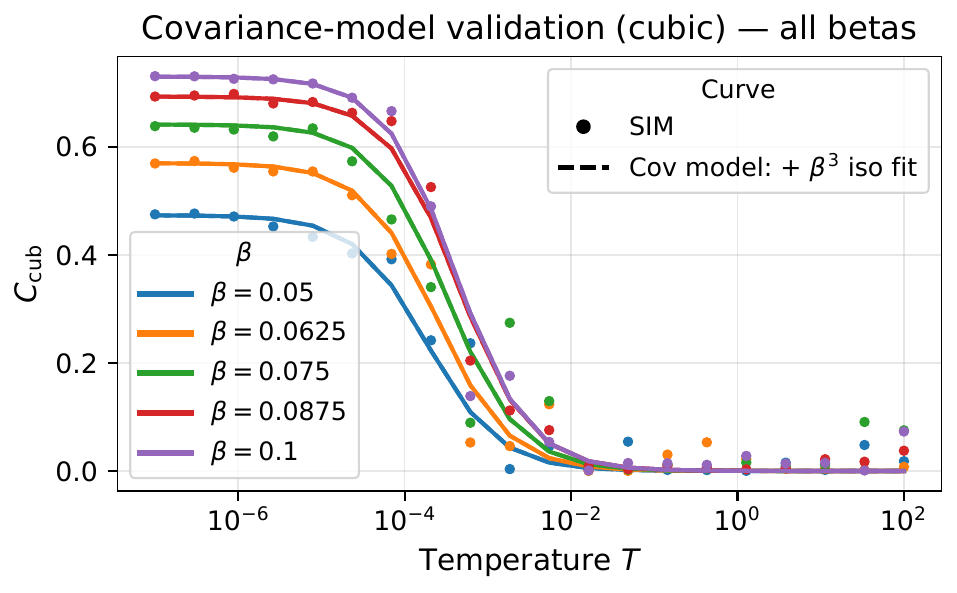}
  \end{minipage}
  \caption{Covariance-fit validation for the Duffing oscillator.
  Deterministic linear and cubic capacities versus $T$ for representative $\beta$ values.
  Markers denote direct simulation estimates of the deterministic sector; curves are covariance-model predictions from Eq.~\eqref{eq:duff-gT} with fitted nonnegative $\{a_k(\beta)\}$ in the isotropic correction of Eq.~\eqref{eq:duff-add-fit}.}
  \label{fig:duff-covfit}
\end{figure*}

\section{Innovation Capacity}
\label{sec:innovation}

Section~\ref{sec:numerical} computed capacities using a finite, sectorized task basis built from input histories and the residual $\Delta X$. We now give a basis-free definition of innovation capacity and prove exact identities. Throughout, expectations and covariances are with respect to a stationary distribution when it exists and all random variables are assumed square-integrable.

We consider a (possibly stochastic) driven dynamical system observed through a linear readout \cite{ehlers2025stochastic}.
At discrete times $t\in\mathbb{Z}^+$ the internal state $s_t$ evolves as
\begin{equation}
\label{eq:reservoir}
s_t \;=\; F(s_{t-1},\,u_t,\,\eta_t),\qquad
X_t \;=\; H(s_t)\in\mathbb{R}^{d},
\end{equation}
where $u_t$ is the external input, $\eta_t$ denotes exogenous fluctuations, $F$ is the dynamical update map, $H$ is the (linear) readout map and $X_t$ collects the $d$ readout coordinates.

Define the full filtration and its input subfiltration by
\begin{equation}
\begin{aligned}
\label{eq:filtrations}
\mathscr F_t \;&:=\; \sigma\!\big(s_0,\ (u_k,\eta_k):\ k\le t\big),
\\
\mathscr F^{\rm in}_t \;&:=\; \sigma(u_k:\ k\le t)\ \subseteq\ \mathscr F_t,
\end{aligned}
\end{equation}
with $\sigma(\cdot)$ denoting the generated sigma-algebra.
(For fading-memory reservoirs one can equivalently replace $\mathscr F^{\rm in}_t$ by a fixed finite window $\sigma(u_{t-h+1:t})$; the basis-free identities below are unchanged \cite{jaeger2001echo, maass2002real, lukosevicius2009reservoir}.)

\begin{equation}
\label{eq:doob-split}
\langle X_t\rangle \;:=\; \E[X_t\mid \mathscr F^{\rm in}_t],\qquad
\Delta X_t \;:=\; X_t-\langle X_t\rangle.
\end{equation}
By construction $\Delta X_t\perp L^2(\mathscr F^{\rm in}_t)$ in $L^2$.

All random variables live in the real Hilbert space
\begin{equation}
\label{eq:L2zero}
L^2_0(\Omega)\ :=\ \Big\{Z\in L^2(\Omega):\ \E[Z]=0\Big\},
\end{equation}
with inner product $\langle A,B\rangle:=\E[AB]$, since capacities are invariant to adding constants and we work with centered tasks/features.

Define the Doob-predictable task subspace and its orthogonal complement:
\begin{equation}
\label{eq:S-and-N}
\mathcal S\ :=\ L^2_0(\mathscr F^{\rm in}_t),\qquad
\mathcal N\ :=\ \mathcal S^\perp.
\end{equation}
Thus $\mathcal S$ consists of centered input-measurable tasks and $\mathcal N$ is the innovation task space (noise-only plus mixed tasks). Let $\mathcal H_X:=\mathrm{span}\{X_{t,1},\dots,X_{t,d}\}\subset L^2_0(\Omega)$ be the readout subspace and let $\Pi_X$ be the $L^2$-orthogonal projector onto $\mathcal H_X$.

\begin{definition}[Projection capacity and sector capacities]
\label{def:sectors}
For $Z\in L^2_0(\Omega)$ define the per-task capacity (with respect to the readout features $X_t$) by
\begin{equation}
\label{eq:per-task-cap}
C_X[Z]\;:=\;\|\Pi_X Z\|_{L^2}^2\;=\;\E\!\big[(\Pi_X Z)^2\big].
\end{equation}
If $Z$ has unit variance, $C_X[Z]$ equals the population $R^2$ of the best linear predictor of $Z$ from $X_t$.

Let $\{T_\ell\}$ be any complete orthonormal basis (ONB) of $\mathcal S$ and $\{U_m\}$ any complete ONB of $\mathcal N$, with all elements centered and unit variance.
Define the predictable and innovation capacities by
\begin{equation}
\label{eq:IPC-IC-def}
\IPC\ :=\ \sum_{\ell} C_X[T_\ell],\qquad
\IC\ :=\ \sum_{m} C_X[U_m].
\end{equation}
Since $\Pi_X$ has rank at most $d$, $\Pi_X\Pi_{\mathcal S}$ and $\Pi_X\Pi_{\mathcal N}$ are trace-class and the sums in \eqref{eq:IPC-IC-def} converge absolutely; Lemma~\ref{lem:trace-capacity} implies they are basis-independent and equal Hilbert-Schmidt traces.
\end{definition}

\begin{lemma}[Trace representation of summed capacities]
\label{lem:trace-capacity}
Let $\{T_\ell\}$ be an ONB for a closed subspace $\mathcal U\subset L^2_0(\Omega)$ with $\E[T_\ell^2]=1$.
Then
\begin{equation}
\sum_{\ell}C_X[T_\ell]=\Tr(\Pi_X\Pi_{\mathcal U}),
\qquad
0\le \Tr(\Pi_X\Pi_{\mathcal U})\le \Tr(\Pi_X),
\end{equation}
where $\Pi_{\mathcal U}$ is the orthogonal projector onto $\mathcal U$.
\end{lemma}

\begin{proof}
Since $\Pi_{\mathcal U}T_\ell=T_\ell$ and $\Pi_X$ is self-adjoint,
\begin{equation}
C_X[T_\ell]=\langle \Pi_X T_\ell,\Pi_X T_\ell\rangle
=\langle T_\ell,\Pi_X T_\ell\rangle
=\langle T_\ell,\Pi_X\Pi_{\mathcal U} T_\ell\rangle.
\end{equation}
Summing over an ONB gives $\sum_{\ell}C_X[T_\ell]=\Tr(\Pi_X\Pi_{\mathcal U})$ (the trace is well-defined since $\Pi_X$ is finite-rank). For the bounds, use cyclicity of trace (again justified by finite rank of $\Pi_X$):
\begin{equation}
\Tr(\Pi_X\Pi_{\mathcal U})=\Tr(\Pi_X\Pi_{\mathcal U}\Pi_X).
\end{equation}
Because $0\preceq \Pi_{\mathcal U}\preceq I$, we have $0\preceq \Pi_X\Pi_{\mathcal U}\Pi_X\preceq \Pi_X$. Taking traces yields $0\le \Tr(\Pi_X\Pi_{\mathcal U})\le \Tr(\Pi_X)$.
\end{proof}

\begin{lemma}[Readout dimension equals covariance rank]
\label{lem:trace-rank}
Let $X=(X^{(1)},\dots,X^{(d)})^\top$ be a centered $\R^d$-valued random vector with covariance $\Sigma_{XX}=\E[XX^\top]$.
Then
\begin{equation}
\Tr(\Pi_X)=\dim\mathcal H_X=\rank\Sigma_{XX}\le d.
\end{equation}
\end{lemma}

\begin{proof}
Define the linear map $A:\R^d\to L^2_0(\Omega)$ by $Aw:=w^\top X=\sum_{j=1}^d w_j X^{(j)}$.
Then $\range(A)=\mathcal H_X$.
Its Hilbert adjoint $A^*:\mathcal H_X\to\R^d$ satisfies $A^*Y=\E[XY]$ for $Y\in\mathcal H_X$, so
\begin{equation}
A^*A w=\E\!\big[X(w^\top X)\big]=\Sigma_{XX}w.
\end{equation}
Hence $\rank(\Sigma_{XX})=\rank(A^*A)=\rank(A)=\dim\range(A)=\dim\mathcal H_X$.
An orthogonal projector has trace equal to the dimension of its range, so $\Tr(\Pi_X)=\dim\mathcal H_X=\rank\Sigma_{XX}\le d$.
\end{proof}

\begin{theorem}[Conservation of observable rank]
\label{thm:conservation}
Let $\mathcal S,\mathcal N\subset L^2_0(\Omega)$ be the Doob-predictable and innovation subspaces defined in \eqref{eq:S-and-N}.
Then
\begin{equation}
\label{eq:conservation-law}
\IPC \;+\; \IC
\;=\; \Tr(\Pi_X)
\;=\; \rank \Sigma_{XX}\ \le\ d,
\end{equation}
where $\Pi_X$ is the $L^2$-orthogonal projector onto $\mathcal H_X=\mathrm{span}\{X_{t,1},\dots,X_{t,d}\}$ and $\Sigma_{XX}=\E[XX^\top]$.
\end{theorem}

\begin{proof}
Lemma~\ref{lem:trace-capacity} gives $\IPC=\Tr(\Pi_X\Pi_{\mathcal S})$ and $\IC=\Tr(\Pi_X\Pi_{\mathcal N})$.
Since $\mathcal S\oplus \mathcal N=L^2_0(\Omega)$, $\Pi_{\mathcal S}+\Pi_{\mathcal N}=I$ on $L^2_0(\Omega)$, hence
\begin{equation}
\IPC+\IC=\Tr\!\big(\Pi_X(\Pi_{\mathcal S}+\Pi_{\mathcal N})\big)=\Tr(\Pi_X).
\end{equation}
Lemma~\ref{lem:trace-rank} gives $\Tr(\Pi_X)=\rank\Sigma_{XX}\le d$.
\end{proof}

In additive settings where $X$ decomposes as an input-only functional plus a noise-only functional (so the mixed sector is absent), $\IC$ can be interpreted as the usual Dambre capacity computed with the noise source treated as the ``signal.'' In general nonlinear reservoirs the mixed sector is nonempty; $\IC$ then includes both noise-only and input$\times$noise tasks.

\begin{corollary}[Innovation allocation in high-rank stochastic reservoirs]
\label{cor:phys-stoch}
Physical, stochastic reservoir computers (\cite{ehlers2025stochastic,polloreno2025restrictions}), defined on bitstring probabilities and furnished with the power set of their readout monomials, have an exponentially large innovation capacity.
\end{corollary}
\begin{proof}
Physical, stochastic reservoir computers defined using bitstring probabilities have only a polynomial amount of $\IPC$, despite their exponentially large number of readout signals (see \cite{polloreno2025restrictions}). By Theorem~\ref{thm:conservation}, the remaining (high) readout rank is necessarily allocated to innovation capacity.
\end{proof}

\section{Explored State Space and Innovation Geometry}
\label{sec:innovation-geometry}

The conservation law in Theorem~\ref{thm:conservation} is a one-step second-moment identity.
This section connects the one-step innovation budget $\IC$ to consequences for geometry in whitened readout space and for the explored set of innovation histories over blocks.
A large $\IC$ forces a large subspace of directions with nontrivial innovation fraction.
Under mild dependence and anti-concentration regularity on that subspace, this width lifts to extensive block entropy and to an average-case lower bound for total-variation learning localized to typical outcomes.

We begin with the Doob decomposition of length-$b$ histories.
Fix $b\in\N$ and define the stacked readout block
\begin{equation}
X_{t-b+1:t}:=[X_{t-b+1}^\top,\ldots,X_t^\top]^\top\in\R^{bd}.
\end{equation}
Stacking the one-step decomposition \eqref{eq:doob-split} yields
\begin{equation}
\label{eq:block-doob}
X_{t-b+1:t}
=
\langle X\rangle_{t-b+1:t}
+
\Delta X_{t-b+1:t},
\end{equation}
where
$\langle X\rangle_{t-b+1:t}:=[\langle X_{t-b+1}\rangle^\top,\ldots,\langle X_t\rangle^\top]^\top$
and
$\Delta X_{t-b+1:t}:=[\Delta X_{t-b+1}^\top,\ldots,\Delta X_t^\top]^\top$
is the innovation block. For the one-step geometry we take $b=1$.

Let $X_t\in\R^{d}$ be centered with covariance $\Sigma_{XX}:=\Cov(X_t)$ and rank $r:=\rank\Sigma_{XX}$.
Define predictable and innovation covariances with respect to $\mathscr F^{\rm in}_t$,
\begin{equation}
\label{eq:S-N-split}
S:=\Cov\!\big(\E[X_t\mid \mathscr F^{\rm in}_t]\big),\qquad
N:=\E\!\big[\Cov(X_t\mid \mathscr F^{\rm in}_t)\big],
\end{equation}
so $\Sigma_{XX}=S+N$.

\subsection{One-step ellipsoid geometry}

Work on the active readout subspace $\mathcal R_X:=\range(\Sigma_{XX})\subset\R^{d}$.
In whitened coordinates,
\begin{equation}
\label{eq:whiten-def}
Z_t:=\Sigma_{XX}^{+/2}X_t\in\mathcal R_X,\qquad \Cov(Z_t)=\Pi_{\mathcal R_X},
\end{equation}
where $\Pi_{\mathcal R_X}$ denotes the Euclidean orthogonal projector onto $\mathcal R_X$.
Define the predictable-fraction operator on $\mathcal R_X$ by
\begin{equation}
\label{eq:Gamma-def-sec}
\Gamma \ :=\ \Sigma_{XX}^{+/2}\,S\,\Sigma_{XX}^{+/2},
\qquad 0\preceq \Gamma \preceq \Pi_{\mathcal R_X}.
\end{equation}

\begin{proposition}[Whitened predictable and innovation ellipsoids]
\label{prop:innovation-ellipsoid}
Let $\gamma_1\ge\dots\ge\gamma_r$ be the eigenvalues of $\Gamma$ on $\mathcal R_X$, counting multiplicity, so $\gamma_k\in[0,1]$.
Define $\Gamma^C:=\Pi_{\mathcal R_X}-\Gamma$.
Then
\[
\Cov(Z^{\rm pred}_t)=\Gamma,\qquad \Cov(Z^{\rm innov}_t)=\Gamma^C,
\]
so the predictable and innovation covariances correspond to axis-aligned ellipsoids with semiaxes
$\{\sqrt{\gamma_k}\}$ and $\{\sqrt{1-\gamma_k}\}$.

Moreover,
\begin{equation}
\label{eq:second}
\IPC=\sum_{k=1}^{r}\gamma_k,\qquad \IC=\sum_{k=1}^{r}(1-\gamma_k),\qquad \IPC+\IC=r.
\end{equation}

Define the (possibly degenerate) covariance ellipsoids in $\mathcal R_X$ by
\[
\mathcal E_{\rm pred}:=\{\Gamma^{1/2}y:\ y\in\mathbb{B}_r\},\qquad
\mathcal E_{\rm innov}:=\{(\Gamma^C)^{1/2}y:\ y\in\mathbb{B}_r\},
\]
whose intrinsic dimensions are $\rank(\Gamma)$ and $\rank(\Gamma^C)$.
Letting $\det_{+}(\cdot)$ denote the pseudo-determinant and letting $\mathbb{B}_k$ denote the Euclidean unit ball in $\R^k$, their intrinsic volumes satisfy
\begin{equation}
\begin{aligned}
&\Vol_{\rank(\Gamma)}(\mathcal E_{\rm pred})=\Vol(\mathbb{B}_{\rank(\Gamma)})\,\det_{+}(\Gamma)^{1/2},\\
&\Vol_{\rank(\Gamma^C)}(\mathcal E_{\rm innov})
=\Vol(\mathbb{B}_{\rank(\Gamma^C)})\,\det_{+}(\Gamma^C)^{1/2}.
\end{aligned}
\end{equation}
\end{proposition}

\begin{proof}
Define
\[
Z_t^{\rm pred}:=\Sigma_{XX}^{+/2}\E[X_t\mid\mathscr F^{\rm in}_t],
\qquad
Z_t^{\rm innov}:=\Sigma_{XX}^{+/2}\Delta X_t.
\]
Because $\Sigma_{XX}=S+N$ and the predictable and innovation split is orthogonal in $L^2$,
\begin{equation}
\label{eq:pred-innov-cov}
\begin{aligned}
\Cov(Z_t^{\rm pred})&=\Sigma_{XX}^{+/2}\,S\,\Sigma_{XX}^{+/2}=\Gamma,\\
\Cov(Z_t^{\rm innov})&=\Sigma_{XX}^{+/2}\,N\,\Sigma_{XX}^{+/2}=\Pi_{\mathcal R_X}-\Gamma=\Gamma^C,
\end{aligned}
\end{equation}
where we used $\Cov(Z_t)=\Pi_{\mathcal R_X}$ from \eqref{eq:whiten-def}.
Diagonalizing $\Gamma$ yields the semiaxis description.

For the trace identities, let $A:\R^{d}\to L^2_0(\Omega)$ be $Aw=w^\top X_t$.
Then $A^*A=\Sigma_{XX}$ and the $L^2$-orthogonal projector onto
$\mathcal H_X=\range(A)$ is $\Pi_X=A\,\Sigma_{XX}^{+}A^{*}$.
Using cyclicity of trace for finite-rank operators,
\[
\Tr(\Pi_X\Pi_{\mathcal S})
=\Tr\!\big(\Sigma_{XX}^{+}\,A^{*}\Pi_{\mathcal S}A\big).
\]
The projector $\Pi_{\mathcal S}$ is conditional expectation onto $\mathscr F^{\rm in}_t$, so by the tower property and $\E[X_t]=0$,
\begin{equation}
\begin{aligned}
A^{*}\Pi_{\mathcal S}A
&= \E\!\big[X_t\,\E[X_t^\top\mid \mathscr F^{\rm in}_t]\big]
\\&= \E\!\big[\E[X_t\mid \mathscr F^{\rm in}_t]\E[X_t^\top\mid \mathscr F^{\rm in}_t]\big]
= S.
\end{aligned}
\end{equation}
Thus $\IPC=\Tr(\Pi_X\Pi_{\mathcal S})=\Tr(S\Sigma_{XX}^{+})=\Tr(\Gamma)=\sum_{k=1}^r\gamma_k$.
Since $\Pi_{\mathcal R_X}-\Gamma$ has eigenvalues $\{1-\gamma_k\}$ on $\mathcal R_X$, we obtain $\IC=\sum_{k=1}^r(1-\gamma_k)$ and $\IPC+\IC=r$.
The volume relations follow from the standard ellipsoid formula on the support subspaces.
\end{proof}

\subsection{A trimmed $\tau$-innovation subspace controlled by $\IC$}

Let $\Gamma v_k=\gamma_k v_k$ be an eigendecomposition on $\mathcal R_X$.
Fix $\tau\in(0,1)$ and define
\begin{equation}
\begin{aligned}
I_\tau&:=\{k\in\{1,\dots,r\}\,:\ 1-\gamma_k\ge \tau\},\qquad
L_\tau:=|I_\tau|,\\
\mathcal U_\tau&:=\mathrm{span}\{v_k\,:\ k\in I_\tau\}\subseteq \mathcal R_X.
\end{aligned}
\end{equation}
Let $P_\tau\in\mathbb{R}^{L_\tau\times d}$ have orthonormal rows spanning $\mathcal U_\tau$.
Let $\Delta Z_t:=\Sigma_{XX}^{+/2}\Delta X_t$ denote the whitened innovation.

\begin{lemma}[$\tau$-subspace variance floor and dimension bounds]
\label{lem:Lt-bounds}
With the notation above,
\begin{equation}\label{eq:tau-floor}
\Cov(P_\tau \Delta Z_t)\;=\;P_\tau (\Pi_{\mathcal R_X}-\Gamma) P_\tau^\top \;\succeq\;\tau I_{L_\tau}.
\end{equation}
Moreover, the subspace dimension satisfies
\begin{equation}\label{eq:Ltau-two-sided}
\max\!\left\{0,\frac{\IC-\tau r}{1-\tau}\right\}\ \le\ L_\tau\ \le\ \frac{\IC}{\tau}.
\end{equation}
\end{lemma}

\begin{proof}
Since $(\Pi_{\mathcal R_X}-\Gamma) v_k=(1-\gamma_k)v_k$, the restriction of $(\Pi_{\mathcal R_X}-\Gamma)$ to $\mathcal U_\tau$ has all eigenvalues at least $\tau$.
For any $x\in\mathbb{R}^{L_\tau}$,
\begin{equation}
\begin{aligned}
x^\top \!\big(P_\tau (\Pi_{\mathcal R_X}-\Gamma) P_\tau^\top\big)x
&=(P_\tau^\top x)^\top (\Pi_{\mathcal R_X}-\Gamma) (P_\tau^\top x)
\\&\ge \tau \|P_\tau^\top x\|_2^2
=\tau\|x\|_2^2,
\end{aligned}
\end{equation}
where we used $\range(P_\tau^\top)=\mathcal U_\tau$ and $P_\tau P_\tau^\top=I_{L_\tau}$.
This proves \eqref{eq:tau-floor}.

For the bounds on $L_\tau$, note that if $\gamma_k>1-\tau$, then index $k$ contributes more than $1-\tau$ to $\IPC=\sum_{j=1}^r\gamma_j$.
Hence there can be at most $\IPC/(1-\tau)$ such indices.
Therefore
\[
L_\tau \ge r-\frac{\IPC}{1-\tau}
=\frac{r-\IPC-\tau r}{1-\tau}
=\frac{\IC-\tau r}{1-\tau},
\]
and taking $\max\{\cdot,0\}$ yields the stated lower bound.
For the upper bound, since each $k\in I_\tau$ contributes at least $\tau$ to $\IC=\sum_{k=1}^r(1-\gamma_k)$,
\begin{equation}
\IC \ge \sum_{k\in I_\tau}(1-\gamma_k)\ge \tau |I_\tau|=\tau L_\tau,
\end{equation}
so $L_\tau\le \IC/\tau$.
\end{proof}

Define the projected one-step whitened innovation on the $\tau$-subspace by
\begin{equation}
\label{eq:Yt-def}
Y_t \ :=\ P_\tau \Delta Z_t \ \in\ \R^{L_\tau}.
\end{equation}
Lemma~\ref{lem:Lt-bounds} gives a one-step covariance floor $\Cov(Y_t)\succeq \tau I_{L_\tau}$ and bounds $L_\tau$, hence $m=L_\tau b$, explicitly in terms of $\IC$.

\subsection{Typical innovation histories}

Fix a block length $b\in\N$ and set $m:=L_\tau b$.
Define the stacked innovation block on the $\tau$-subspace,
\begin{equation}
Y^{(b)}_t := [\,Y_{t-b+1},\dots,Y_t\,]\in\R^{m}.
\end{equation}
A one-step covariance floor does not by itself prevent temporally stale innovation blocks.
To lift one-step width into a block-level statement we impose a weak-dependence condition in Appendix~\ref{app:assumptions} that keeps the block covariance well-conditioned.
To convert a covariance floor into a differential-entropy floor, we also impose an anti-concentration regularity through a bounded isotropic constant.

\begin{theorem}[Exponential growth of distinguishable innovation histories]
\label{thm:innovation-exp}
Work under Assumptions~\ref{asmp:innov-regularity} and~\ref{asmp:mixing}.
Fix $\tau\in(0,1)$ and let $P_\tau$ be the $\tau$-innovation projector from Lemma~\ref{lem:Lt-bounds} with rank $L_\tau$. Fix a block length $b\in\N$ and set $m:=L_\tau b$.
Let $\Delta Z_t:=\Sigma_{XX}^{+/2}\Delta X_t$ and define the projected whitened innovation $Y_t:=P_\tau\Delta Z_t\in\R^{L_\tau}$ and the stacked length-$b$ innovation block
\begin{equation}
Y^{(b)}_t := [\,Y_{t-b+1},\dots,Y_t\,]\in\R^{m}.
\end{equation}
Then $\Cov(Y^{(b)}_t)\succeq (\tau/2)\,I_m$ for all $t$ and
\begin{equation}
\label{eq:block-entropy-lb}
h(Y^{(b)}_t)
\ \ge\
\frac{m}{2}\log\!\Big(\frac{\tau}{2L_\star^2}\Big),
\end{equation}
where $L_\star$ is the isotropic-constant bound from Assumption~\ref{asmp:innov-regularity}.
Moreover, if the asymptotic equipartition property (AEP) part of Assumption~\ref{asmp:mixing} holds for the stationary process of innovation blocks, then for any resolution $\rho\in(0,1)$ and any $(1-\epsilon)$-typical set $\mathcal T$ of $Y^{(b)}_t$, the covering number satisfies
\begin{equation}
\label{eq:covering-lb}
\log N_\rho\big(\mathcal T\big)
\ \ge\
\frac{m}{2}\log\!\Big(\frac{\tau}{2L_\star^2}\Big)
\ +\
m\,\log(1/\rho)
\ -\ O(m).
\end{equation}
Up to subexponential factors, the number of $\rho$-distinguishable innovation histories scales as $\exp(h(Y^{(b)}_t))\rho^{-m}$.
\end{theorem}

\begin{proof}
Assumption~\ref{asmp:innov-regularity} gives absolute continuity of $Y^{(b)}_t$ and the isotropic-constant bound $L_{Y^{(b)}_t}\le L_\star$. Lemma~\ref{lem:Lt-bounds} gives $\Cov(Y_t)\succeq \tau I_{L_\tau}$. Assumption~\ref{asmp:mixing} implies $\sum_{k\ge 1}\|\Cov(Y_0,Y_k)\|_{\mathrm{op}}\le \tau/4$, so Lemma~\ref{lem:block-cov-floor} yields the uniform block covariance floor
\begin{equation}
\Cov(Y^{(b)}_t)\ \succeq\ \frac{\tau}{2}\,I_m
\qquad\text{for all }t.
\end{equation}
Applying Proposition~\ref{prop:isotropic-entropy} with $\sigma^2=\tau/2$ yields \eqref{eq:block-entropy-lb}.
For \eqref{eq:covering-lb}, the AEP implies a typical set $\mathcal T$ with $\log\Vol(\mathcal T)=h(Y^{(b)}_t)\pm O(m)$.
Covering $\mathcal T$ by Euclidean balls of radius $\rho$ yields
\begin{equation}
N_\rho(\mathcal T)\ \gtrsim\ \frac{\Vol(\mathcal T)}{\Vol(\mathbb{B}^{m}_\rho)}
= \Vol(\mathcal T)\cdot \rho^{-m}\cdot \Vol(\mathbb{B}^{m}_1)^{-1},
\end{equation}
hence $\log N_\rho(\mathcal T)\ge \log\Vol(\mathcal T)+m\log(1/\rho)-O(m)$ and substituting the entropy bound gives \eqref{eq:covering-lb}.
\end{proof}

We now connect the effective dimension $m=L_\tau b$ to distribution learning.
Given $n$ independent samples from an unknown law $P$ on $\R^m$, an estimator $\widehat P$ seeks to approximate $P$ in total variation. The regularity assumptions used to justify typical-set geometry rule out atomic innovation-block laws and the packing below is localized inside a typical set.

\begin{theorem}[Typical-set-localized total-variation hardness]
\label{thm:avgcase-innovation}
Assume Assumptions~\ref{asmp:innov-regularity} and~\ref{asmp:mixing}.
Fix $b\in\N$ and $\tau\in(0,1)$, let $P_\tau$ be the $\tau$-innovation projector from Lemma~\ref{lem:Lt-bounds},
and set $m:=L_\tau b$.
Let $W\in\R^{m\times m}$ be invertible and set $Y:=WY^{(b)}_t$, where
\begin{equation}
\label{eq:Yblock-def}
\begin{aligned}
Y^{(b)}_t &:= [\,P_\tau \Delta Z_{t-b+1},\dots,P_\tau \Delta Z_t\,]\in\mathbb{R}^m,
\\
\Delta Z_t&:=\Sigma_{XX}^{+/2}\Delta X_t.
\end{aligned}
\end{equation}
Let $P_0$ be the law of $Y$.
Let $\mathcal T$ be any $(1-\epsilon)$-typical set for $Y$ provided by the AEP in Assumption~\ref{asmp:mixing}, so $P_0(\mathcal T)\ge 1-\epsilon$. Then there exist universal constants $c_0,c_1,c_2,c_3>0$ such that for every $\alpha\in(0,1/2]$ one can construct a set
$\mathcal V\subset\{\pm1\}^m$ with $|\mathcal V|\ge \exp(c_0 m)$ and a family of laws $\{P_v\}_{v\in \mathcal V}$ on $\R^m$ such that
\begin{enumerate}
\item Typical-set localization. For all $v\in \mathcal V$, $P_v(\mathcal T)=P_0(\mathcal T)\ge 1-\epsilon$, $dP_v/dP_0=1$ on $\mathcal T^c$ and $dP_v/dP_0\in[1-\alpha,1+\alpha]$ on $\mathcal T$.
\item Total-variation separation. For all $v\neq v'$,
$\|P_v-P_{v'}\|_{\mathrm{TV}}\ge c_1(1-\epsilon)\alpha$.
\item Kullback-Leibler closeness. For all $v\neq v'$,
$D_{\mathrm{KL}}(P_v\|P_{v'})\le c_2(1-\epsilon)\alpha^2$.
\end{enumerate}
Consequently, if $V$ is drawn uniformly from $\mathcal V$ and $Y_1,\dots,Y_n$ are independent samples from $P_V$, then every estimator $\widehat v$
satisfies the average-case error bound
\begin{equation}
\Pr\{\widehat v\neq V\}
\;\ge\;
1-\frac{n\,c_2(1-\epsilon)\alpha^2+\log 2}{c_0 m}.
\end{equation}
Moreover, for any estimator $\widehat P$ of the law in total variation based on $n$ independent samples,
\begin{equation}
\E\big[\|\widehat P-P_V\|_{\mathrm{TV}}\big]
\;\ge\;
c_3(1-\epsilon)\alpha\left(1-\frac{n\,c_2(1-\epsilon)\alpha^2+\log 2}{c_0 m}\right).
\end{equation}
In particular, to make the average total-variation error $o(\alpha)$ uniformly over this typical-set-localized family, one needs $n=\Omega(m/\alpha^2)=\Omega(L_\tau b/\alpha^2)$.
\end{theorem}

\begin{proof}
Assumption~\ref{asmp:innov-regularity} implies that $P_0$ is non-atomic by Lemma~\ref{lem:nonatomic}, so the conditional law
$P_0(\cdot\mid \mathcal T)$ is also non-atomic.
Let $p:=P_0(\mathcal T)\ge 1-\epsilon$.
By Lemma~\ref{lem:equal-mass-partition} with $M=2m$, there exist disjoint sets
\begin{equation}
A_1^+,A_1^-,\ldots,A_m^+,A_m^- \subset \mathcal T
\quad\text{with}\quad
P_0(A_i^+)=P_0(A_i^-)=\frac{p}{2m}.
\end{equation}
Define for each $v\in\{\pm1\}^m$ the Radon-Nikodym derivative~\cite{nikodym1930generalisation}
\begin{equation}
\frac{dP_v}{dP_0}(y)
:=
\1_{\mathcal T^c}(y)
+
\sum_{i=1}^m\Big((1+\alpha v_i)\1_{A_i^+}(y) + (1-\alpha v_i)\1_{A_i^-}(y)\Big).
\end{equation}
Because the cells partition $\mathcal T$ and $(1+\alpha v_i)+(1-\alpha v_i)=2$, the integral over $\mathcal T$ equals $p$.
On $\mathcal T^c$ the density equals $1$.
Therefore $\int dP_v=1$.
This also proves localization.

Let $f_v:=dP_v/dP_0$ and write $d_H(v,v')$ for Hamming distance.
If $v_i\neq v_i'$, then on each of $A_i^+$ and $A_i^-$ we have $|f_v-f_{v'}|=2\alpha$ pointwise.
Otherwise the contribution is zero.
Thus
\begin{equation}
\|P_v-P_{v'}\|_{\rm TV}
=\frac12\int |f_v-f_{v'}|\,dP_0
=\frac{\alpha}{m}\,d_H(v,v')\,p.
\end{equation}
Similarly, on indices where $v_i\neq v_i'$,
\begin{equation}
\begin{aligned}
D_{\rm KL}(P_v\|P_{v'})
&=\int f_v\log\frac{f_v}{f_{v'}}\,dP_0
\\&=\frac{d_H(v,v')}{m}\,p\,\alpha\,\log\!\Big(\frac{1+\alpha}{1-\alpha}\Big)
\\&\le 4p\,\alpha^2\,\frac{d_H(v,v')}{m},
\end{aligned}
\end{equation}
using $\alpha\log\!\big(\tfrac{1+\alpha}{1-\alpha}\big)\le 4\alpha^2$ for $\alpha\in(0,1/2]$. By the Varshamov-Gilbert bound \cite{varshamov1957estimate}, there exists $\mathcal V\subset\{\pm1\}^m$ with $|\mathcal V|\ge \exp(c_0 m)$ and $d_H(v,v')\ge m/4$ for $v\neq v'$.
Restricting to this set yields the stated total-variation and Kullback-Leibler bounds with constants scaled by $p\ge 1-\epsilon$.

Now draw $V$ uniformly from $\mathcal V$ and let $Y_{1:n}$ be independent samples from $P_V$.
Fano's inequality gives
\begin{equation}
\Pr\{\widehat v\neq V\}
\ \ge\
1-\frac{I(V;Y_{1:n})+\log 2}{\log |\mathcal V|}.
\end{equation}
Using $I(V;Y_{1:n})\le n\,\max_{v\neq v'} D_{\mathrm{KL}}(P_v\|P_{v'}) \le n\,c_2(1-\epsilon)\alpha^2$
and $\log|\mathcal V|\ge c_0 m$ gives the stated average-case testing lower bound. To lower bound total-variation estimation risk, define a classifier from any estimator $\widehat P$ by
\begin{equation}
\widehat v(\widehat P)\in\arg\min_{v\in \mathcal V}\ \|\widehat P-P_v\|_{\rm TV}.
\end{equation}
If $\|\widehat P-P_V\|_{\rm TV}<\tfrac12\min_{v'\neq V}\|P_V-P_{v'}\|_{\rm TV}$ then necessarily $\widehat v(\widehat P)=V$.
Let $\Delta:=\min_{v\neq v'}\|P_v-P_{v'}\|_{\rm TV}\ge c_1(1-\epsilon)\alpha$.
Then
\begin{equation}
\begin{aligned}
\Pr\{\widehat v(\widehat P)\neq V\}
&\le
\Pr\{\|\widehat P-P_V\|_{\rm TV}\ge \Delta/2\}
\\&\le
\frac{2}{\Delta}\,\E[\|\widehat P-P_V\|_{\rm TV}],
\end{aligned}
\end{equation}
by Markov's inequality.
Rearranging yields $\E[\|\widehat P-P_V\|_{\rm TV}]
\ \ge\
\frac{\Delta}{2}\,\Pr\{\widehat v(\widehat P)\neq V\}$ and substituting the testing lower bound completes the proof.
\end{proof}

\section{Conclusion}

This work introduced the innovation capacity $\IC$ of a dynamical learning system as a complement to the information-processing, Doob-predictable capacity $\IPC$.
For a $d$-dimensional readout, the observable rank $\rank(\Sigma_{XX})$ splits exactly into predictable and innovation contributions, giving the conservation law $\IPC+\IC=\rank(\Sigma_{XX})\le d$ in Theorem~\ref{thm:conservation}.
Any degradation of $\IPC$ in noisy physical reservoirs is not lost capacity, but is reallocated to tasks orthogonal to the input filtration.

In linear-Gaussian Johnson-Nyquist regimes, the split admits a generalized-eigenvalue shrinkage interpretation and yields an explicit temperature tradeoff in Proposition~\ref{prop:temperature-clean}. Increasing temperature monotonically shifts capacity from $\IPC$ to $\IC$ while conserving their sum. Geometrically, in whitened coordinates the predictable and innovation split corresponds to complementary ellipsoids whose axis fractions sum to one. The quantity $\IC$ is the trace of the innovation fraction operator and therefore quantifies a whitened innovation-volume budget. A nontrivial innovation budget forces a high-dimensional $\tau$-innovation subspace $\mathcal U_\tau$ with an $O(1)$ one-step innovation variance floor by Lemma~\ref{lem:Lt-bounds}.

To lift this one-step width to a block-level explored-set statement, we imposed weak dependence and anti-concentration regularity on the $\tau$-subspace. Under these assumptions, innovation blocks have an explicit extensive differential-entropy lower bound and hence exponentially many distinguishable innovation histories at fixed resolution in Theorem~\ref{thm:innovation-exp}.
Finally, using a typical-set-localized packing in total variation and Kullback-Leibler divergence together with Fano's inequality, Theorem~\ref{thm:avgcase-innovation} shows that learning the induced innovation-block law in total variation is information-theoretically hard on average over an exponentially large family of perturbations supported on typical outcomes.
The effective dimension is $m=L_\tau b$ and Lemma~\ref{lem:Lt-bounds} controls $L_\tau$ explicitly in terms of $\IC$.

\begin{acknowledgments}
AMP thanks Chloe Rossin, André Melo and Eric Peterson for helpful conversations.
\end{acknowledgments}

\appendix
\section{Technical assumptions}
\label{app:assumptions}

\begin{assumption}[Regular innovation blocks on a $\tau$-innovation subspace]
\label{asmp:innov-regularity}
Fix $\tau\in(0,1)$. For each block size $b\in\mathbb{N}$ in the regime of interest, let $P_\tau$ be the $\tau$-innovation projector from Lemma~\ref{lem:Lt-bounds} with rank $L_\tau$ and set $m:=L_\tau b$. Assume the stacked projected whitened innovation block
\begin{equation}
Y^{(b)}_t := [\,P_\tau \Delta Z_{t-b+1},\dots,P_\tau \Delta Z_t\,]\in\mathbb{R}^m
\end{equation}
admits a density $f_{Y^{(b)}_t}$ with finite supremum, $\|f_{Y^{(b)}_t}\|_\infty<\infty$. Assume its isotropic constant
\begin{equation}
L_{Y^{(b)}_t}:=\|f_{Y^{(b)}_t}\|_\infty^{1/m}\,\det(\Cov(Y^{(b)}_t))^{1/(2m)}
\end{equation}
is uniformly bounded by a constant $L_\star<\infty$, independent of $b$ and $t$ in the regime of interest.
\end{assumption}

\begin{assumption}[Weak dependence and AEP on the $\tau$-innovation subspace]
\label{asmp:mixing}
Fix $\tau\in(0,1)$ and let $Y_t:=P_\tau\Delta Z_t\in\R^{L_\tau}$ denote the projected whitened innovation process. Assume $\{Y_t\}$ is stationary. Assume the autocovariances are summable in operator norm,
\begin{equation}
\sum_{k=1}^{\infty}\big\|\Cov(Y_0,Y_k)\big\|_{\mathrm{op}}\ \le\ \frac{\tau}{4}.
\end{equation}
A sufficient condition is quantitative strong mixing with suitable rate, see Lemma~\ref{lem:cov-decay-mixing} and Corollary~\ref{cor:block-floor-from-mixing}; see also standard treatments of Markov-process mixing \cite{meyn2009markov,levin2017markov}.

For each block size $b$, assume the stationary process of stacked innovation blocks
\begin{equation}
Y^{(b)}_t := [\,Y_{t-b+1},\dots,Y_t\,]\in\mathbb{R}^{L_\tau b}
\end{equation}
satisfies a continuous Shannon-McMillan-Breiman asymptotic equipartition property (AEP)\cite{mcmillan1953basic,breiman1957individual,barron1985ergodic,cover1999elements,csiszar2011information}.
For every $\epsilon\in(0,1)$ there exists a $(1-\epsilon)$-typical set $\mathcal T\subset\R^{L_\tau b}$ with
\begin{equation}
\log \Vol(\mathcal T)=h(Y^{(b)}_t)\pm O(L_\tau b),
\end{equation}
where the implicit constant in $O(L_\tau b)$ may depend on $\epsilon$ but not on $b$.
\end{assumption}

\begin{lemma}[Block covariance floor from summable autocovariances]
\label{lem:block-cov-floor}
Let $\{Y_t\}_{t\in\mathbb{Z}}$ be centered and stationary in $\R^{L}$ with
\begin{equation}
\Cov(Y_t)=K_0\succeq \tau I_L
\end{equation}
for some $\tau>0$.
Let $K_k:=\Cov(Y_0,Y_k)$.
Assume the positive-lag autocovariances are absolutely summable in operator norm,
\begin{equation}
\sum_{k=1}^{\infty}\big\|K_k\big\|_{\mathrm{op}}\ \le\ \varepsilon
\qquad\text{for some }\varepsilon\in\bigl(0,\tfrac{\tau}{2}\bigr).
\end{equation}
Then for every block length $b\in\N$,
\begin{equation}\label{eq:block-cov-floor-clean}
\Cov\!\big([Y_{t-b+1},\ldots,Y_t]\big)\ \succeq\ (\tau-2\varepsilon)\,I_{Lb}.
\end{equation}
In particular, if $\sum_{k\ge 1}\|K_k\|_{\mathrm{op}}\le \tau/4$ then
$\Cov([Y_{t-b+1},\ldots,Y_t])\succeq (\tau/2)\,I_{Lb}$ for all $b\in\N$.
\end{lemma}

\begin{proof}
Fix $b\in\N$ and write the stacked block as
\[
Y^{(b)}_t := [Y_{t-b+1}^\top,\ldots,Y_t^\top]^\top \in \R^{Lb}.
\]
Let $x=(x_1^\top,\ldots,x_b^\top)^\top\in\R^{Lb}$ with blocks $x_i\in\R^{L}$.
By stationarity, $\Cov(Y^{(b)}_t)$ is block Toeplitz with diagonal blocks $K_0$ and off-diagonal blocks $K_{j-i}$, hence
\begin{equation}
\begin{aligned}
x^\top \Cov(Y^{(b)}_t)\,x
&=
\sum_{i=1}^b x_i^\top K_0 x_i
\\&+\;
\sum_{1\le i<j\le b}\Big(x_i^\top K_{j-i} x_j \;+\; x_j^\top K_{j-i}^\top x_i\Big).
\end{aligned}
\end{equation}
Using $K_0\succeq \tau I_L$ and the bounds
\begin{equation}
\big|x_i^\top K_{j-i}x_j\big|
\ \le\ \|K_{j-i}\|_{\mathrm{op}}\ \|x_i\|_2\ \|x_j\|_2,
\qquad
2ab\le a^2+b^2,
\end{equation}
we obtain
\begin{equation}
\begin{aligned}
x^\top \Cov(Y^{(b)}_t)\,x
\ &\ge\
\tau\sum_{i=1}^b \|x_i\|_2^2
\\&\quad-\;
2\sum_{k=1}^{b-1}\|K_k\|_{\mathrm{op}}
\sum_{i=1}^{b-k}\|x_i\|_2\,\|x_{i+k}\|_2
\\
&\ge\
\tau\sum_{i=1}^b \|x_i\|_2^2
\\&\quad-\;
\sum_{k=1}^{b-1}\|K_k\|_{\mathrm{op}}
\sum_{i=1}^{b-k}\Big(\|x_i\|_2^2+\|x_{i+k}\|_2^2\Big)
\\
&\ge\
\Big(\tau-2\sum_{k=1}^{b-1}\|K_k\|_{\mathrm{op}}\Big)\sum_{i=1}^b \|x_i\|_2^2
\\
&\ge\
(\tau-2\varepsilon)\,\|x\|_2^2,
\end{aligned}
\end{equation}
where we used $\sum_{i=1}^{b-k}(\|x_i\|_2^2+\|x_{i+k}\|_2^2)\le 2\sum_{i=1}^b\|x_i\|_2^2$.
Since this holds for all $x$, we conclude \eqref{eq:block-cov-floor-clean}.
\end{proof}

\begin{lemma}[Mixing implies operator-norm covariance decay]
\label{lem:cov-decay-mixing}
Let $\{Y_t\}_{t\in\mathbb Z}$ be a centered, stationary process in $\R^{L}$. Let $\alpha_Y(k)$ denote its strong mixing coefficients,
\begin{equation}
\begin{aligned}
\label{eq:alpha-mixing}
\alpha_Y(k)
&:=\sup\big\{\big|\Pr(A\cap B)-\Pr(A)\Pr(B)\big|:\,
\\&A\in\sigma(Y_s\,:\,s\le 0),\ B\in\sigma(Y_s\,:\,s\ge k)\big\}.
\end{aligned}
\end{equation}
Assume there exists $\delta>0$ such that the directional $(2+\delta)$ moment is finite,
\begin{equation}
M_{2+\delta}\;:=\;\sup_{\|v\|_2=1}\ \|v^\top Y_0\|_{L^{2+\delta}}
\;<\;\infty.
\end{equation}
Write $K_k:=\Cov(Y_0,Y_k)\in\mathbb{R}^{L\times L}$.
Then for every $k\ge 1$,
\begin{equation}\label{eq:opcov-alpha}
\|K_k\|_{\mathrm{op}}
\ \le\
8\,M_{2+\delta}^2\ \alpha_Y(k)^{\delta/(2+\delta)}.
\end{equation}
Consequently,
\begin{equation}\label{eq:opcov-sum-alpha}
\sum_{k=1}^\infty \|K_k\|_{\mathrm{op}}
\ \le\
8\,M_{2+\delta}^2\ \sum_{k=1}^\infty \alpha_Y(k)^{\delta/(2+\delta)},
\end{equation}
whenever the right-hand side is finite.
\end{lemma}

\begin{proof}
Fix unit vectors $a,b\in\mathbb R^{L}$ and set $U:=a^\top Y_0$, $V:=b^\top Y_k$.
Then $U$ is measurable with respect to $\sigma(Y_s\,:\,s\le 0)$ and $V$ is measurable with respect to $\sigma(Y_s\,:\,s\ge k)$.
A standard covariance inequality for strong mixing sequences, see \cite{doukhan1994mixing,bradley2005survey}, gives, for $p=q=2+\delta$,
\begin{equation}
\begin{aligned}
|\Cov(U,V)|
\ &\le\ 8\ \alpha_Y(k)^{1-\frac1p-\frac1q}\ \|U\|_{L^{p}}\ \|V\|_{L^{q}}
\\&=
8\ \alpha_Y(k)^{\delta/(2+\delta)}\ \|U\|_{L^{2+\delta}}\ \|V\|_{L^{2+\delta}}.
\end{aligned}
\end{equation}
By definition of $M_{2+\delta}$ and stationarity, $\|U\|_{L^{2+\delta}},\|V\|_{L^{2+\delta}}\le M_{2+\delta}$, hence
\begin{equation}
|a^\top K_k b|
=|\Cov(a^\top Y_0,b^\top Y_k)|
\le 8 M_{2+\delta}^2\alpha_Y(k)^{\delta/(2+\delta)}.
\end{equation}
Taking the supremum over $\|a\|_2=\|b\|_2=1$ yields \eqref{eq:opcov-alpha} and summing gives \eqref{eq:opcov-sum-alpha}.
\end{proof}

\begin{corollary}[A mixing-rate condition sufficient for a uniform block floor]
\label{cor:block-floor-from-mixing}
In the setting of Lemma~\ref{lem:cov-decay-mixing}, assume additionally that $\Cov(Y_0)\succeq \tau I_L$ for some $\tau>0$.
If
\begin{equation}\label{eq:alpha-sum-sufficient}
\sum_{k=1}^\infty \alpha_Y(k)^{\delta/(2+\delta)}
\ \le\
\frac{\tau}{32\,M_{2+\delta}^2},
\end{equation}
then $\sum_{k\ge 1}\|K_k\|_{\mathrm{op}}\le \tau/4$, hence Lemma~\ref{lem:block-cov-floor} yields
\[
\Cov\!\big([Y_{t-b+1},\ldots,Y_t]\big)\ \succeq\ \frac{\tau}{2}\,I_{Lb}
\qquad\text{for all }b\in\N.
\]
\end{corollary}

\begin{lemma}[Non-atomicity under absolute continuity]
\label{lem:nonatomic}
Assume Assumption~\ref{asmp:innov-regularity}.
Let $b\in\mathbb{N}$ and set $m:=L_\tau b$.
Define the stacked projected whitened innovation block
\begin{equation}
Y^{(b)}_t := [\,P_\tau \Delta Z_{t-b+1},\dots,P_\tau \Delta Z_t\,]\in\mathbb{R}^m.
\end{equation}
Then $Y^{(b)}_t$ admits a density with respect to Lebesgue measure on $\mathbb{R}^m$, hence its law is non-atomic.
Moreover, for any invertible $W\in\mathbb{R}^{m\times m}$, the transformed block $Y:=WY^{(b)}_t$ also admits a density and is non-atomic.
\end{lemma}

\begin{proof}
Assumption~\ref{asmp:innov-regularity} gives absolute continuity of $Y^{(b)}_t$, hence non-atomicity.
If $Y=WY^{(b)}_t$ with $W$ invertible, then $Y$ is the pushforward of $Y^{(b)}_t$ by a $C^1$ bijection, so absolute continuity is preserved.
\end{proof}

\begin{lemma}[Equal-mass partition for non-atomic measures]
\label{lem:equal-mass-partition}
Let $P$ be a non-atomic probability measure on $\R^m$ and let $T\subset\R^m$ be measurable with $P(T)=p>0$.
Then for every integer $M\ge 1$ there exist measurable, pairwise disjoint sets $B_1,\dots,B_M\subset T$
with $\bigcup_{j=1}^M B_j\subseteq T$ and
\begin{equation}
P(B_j)=\frac{p}{M}\qquad\text{for all }j=1,\dots,M.
\end{equation}
\end{lemma}

\begin{proof}
Because $P$ is non-atomic, for any $q\in(0,p)$ there exists a measurable subset $B\subset T$ with $P(B)=q$. Construct $B_1$ with $P(B_1)=p/M$. Then $P(T\setminus B_1)=p-p/M=(M-1)p/M$ and $P$ restricted to $T\setminus B_1$ is still non-atomic, so the construction can be repeated to find $B_2\subset T\setminus B_1$ with $P(B_2)=p/M$ and so on.
After $M$ steps the sets are disjoint and have the desired masses.
\end{proof}

\begin{proposition}[Entropy lower bound from isotropic constant]
\label{prop:isotropic-entropy}
Let $Y\in\R^{m}$ admit a density $f_Y$ with $\|f_Y\|_\infty<\infty$ and covariance $K_Y:=\Cov(Y)$. Define the isotropic constant of $Y$ by
\begin{equation}
\label{eq:isotropic-const}
L_Y \;:=\; \|f_Y\|_\infty^{1/m}\,\det(K_Y)^{1/(2m)}.
\end{equation}
Then
\begin{equation}
\begin{aligned}
\label{eq:isotropic-entropy}
h(Y)\ &\ge\ \frac12\log\det(K_Y)\ -\ m\log L_Y
\ \\&=\ \frac{m}{2}\log\!\Big(\frac{\det(K_Y)^{1/m}}{L_Y^2}\Big).
\end{aligned}
\end{equation}
In particular, if $K_Y\succeq \sigma^2 I_m$ and $L_Y\le L_\star$, then
\begin{equation}
\label{eq:isotropic-entropy-sigma}
h(Y)\ \ge\ \frac{m}{2}\log\!\Big(\frac{\sigma^2}{L_\star^2}\Big).
\end{equation}
\end{proposition}

\begin{proof}
Since $f_Y(y)\le \|f_Y\|_\infty$ almost everywhere, we have $-\log f_Y(Y)\ge -\log\|f_Y\|_\infty$ almost surely.
Taking expectations yields
\begin{equation}
h(Y)=\E[-\log f_Y(Y)]\ \ge\ -\log\|f_Y\|_\infty.
\end{equation}
By definition \eqref{eq:isotropic-const}, $\|f_Y\|_\infty = L_Y^m\,\det(K_Y)^{-1/2}$, hence
\begin{equation}
-\log\|f_Y\|_\infty\ =\ \frac12\log\det(K_Y)\ -\ m\log L_Y,
\end{equation}
which gives \eqref{eq:isotropic-entropy}.
If $K_Y\succeq \sigma^2 I_m$, then $\det(K_Y)^{1/m}\ge \sigma^2$, yielding \eqref{eq:isotropic-entropy-sigma}.
\end{proof}

\bibliographystyle{apsrev4-2}
\bibliography{refs}

\end{document}